\documentclass{article}
\PassOptionsToPackage{numbers}{natbib}

\usepackage[final]{nips_2016}

\usepackage[utf8]{inputenc} 
\usepackage[T1]{fontenc}    
\usepackage{hyperref}       
\usepackage{url}            
\usepackage{booktabs}       
\usepackage{amsfonts}       
\usepackage{nicefrac}       
\usepackage{microtype}      

\usepackage{amsmath}
\usepackage{amsfonts}
\usepackage{amsthm}
\usepackage{algorithm}
\usepackage{algorithmic}
\usepackage{graphicx}

\usepackage{times}
\usepackage{graphicx} 
\usepackage{subfigure} 

\usepackage{algorithm}
\usepackage{algorithmic}

\usepackage{hyperref}

\newtheorem{theorem}{Theorem}
\newtheorem{corollary}{Corollary}
\newtheorem{lemma}{Lemma}

\graphicspath{FiguresTalk/}

\title{Nonstationary Distance Metric Learning}

            \author{
  Kristjan Greenewald  \\
  EECS Department\\
  University of Michigan\\
  Ann Arbor, MI 48109 \\
  \texttt{greenewk@umich.edu} \\
  \And
  Stephen Kelley \\
  MIT Lincoln Laboratory \\
  Lexington, MA 02420 \\
  \texttt{stephen.kelley@ll.mit.edu} \\
  \And
  Alfred O. Hero III \\
 EECS Department\\
  University of Michigan\\
  Ann Arbor, MI 48109 \\
  \texttt{hero@umich.edu} \\
  }

\begin{document}
\maketitle

\begin{abstract}
Recent work in distance metric learning has focused on learning transformations of data that best align with provided sets of pairwise similarity and dissimilarity constraints.  
The learned transformations lead to improved retrieval, classification, and clustering algorithms due to the better adapted distance or similarity measures. Here, we introduce the problem of learning these transformations when the underlying constraint generation process is nonstationary. This nonstationarity can be due to changes in either the ground-truth clustering used to generate constraints or changes to the feature subspaces in which the class structure is apparent. We propose and evaluate COMID-SADL, an adaptive, online approach for learning and tracking optimal metrics as they change over time that is highly robust to a variety of nonstationary behaviors in the changing metric. We demonstrate COMID-SADL on both real and synthetic data sets and show significant performance improvements relative to previously proposed batch and online distance metric learning algorithms.
\end{abstract}

\section{Introduction}
%
%
%

The effectiveness of many machine learning and data mining algorithms depends on an appropriate measure of pairwise distance between data points that accurately reflects the learning task, e.g., prediction, clustering or classification. The kNN classifier,  K-means clustering, and the Laplacian-SVM semi-supervised classifier are examples of such {\em distance-based} machine learning algorithms. In settings where there is clean, appropriately-scaled spherical Gaussian data, standard Euclidean distance can be utilized.  However, when the data is heavy tailed, multimodal, or contaminated by outliers, observation noise, or irrelevant or replicated features, use of Euclidean inter-point distance can be problematic, leading to bias or loss of discriminative power. 

To reduce bias and loss of discriminative power of distance-based machine learning algorithms, data-driven approaches for optimizing the distance metric have been proposed. These methodologies, generally taking the form of dimensionality reduction or data ``whitening", aim to utilize the data itself to learn a transformation of the data that embeds it into a space where Euclidean distance is appropriate. Examples of such techniques include Principal Component Analysis \cite{bishop2006pattern}, Multidimensional Scaling \cite{hastie2005elements}, covariance estimation \cite{hastie2005elements,bishop2006pattern}, and manifold learning \cite{lee2007nonlinear}. Such unsupervised methods do not exploit human input on the distance metric, and they overly rely on prior assumptions, e.g., local linearity or smoothness.


In distance metric learning one seeks to learn transformations of the data that are well matched to a particular task specified by the user. Point labels or constraints indicating point similarity or dissimilarity are used to learn a transformation of the data such that similar points are ``close" to one another and dissimilar points are distant in the transformed space.  Learning distance metrics in this manner allows a more precise notion of distance or similarity to be defined that is related to the task at hand.

Many supervised and semi-supervised distance metric learning approaches have been developed \cite{kulis2012metric}. This includes online algorithms \cite{kunapuli2012mirror} with regret guarantees for situations where similarity constraints are received sequentially. 

This paper proposes a new method that provides distance metric tracking. 
Specifically, we suppose the underlying ground-truth (or optimal) distance metric from which constraints are generated is evolving over time, in an unknown and potentially nonstationary way. We propose an adaptive, online approach to track the underlying metric as the constraints are received. Our algorithm, which we call COMID-Strongly Adaptive Dynamic Learning (COMID-SADL) is inspired by recent advances in composite objective mirror descent for metric learning \cite{duchi2010composite} (COMID) and the Strongly Adaptive Online Learning (SAOL) framework proposed in \cite{daniely2015strongly}. We prove strong bounds on the dynamic regret of every subinterval, guaranteeing strong adaptivity and robustness to nonstationary metric drift such as discrete shifts, slow drift with a nonstationary drift rate, and combinations thereof.

\subsection{Related Work} \label{sec:related}


Linear Discriminant Analysis (LDA) and Principal Component Analysis (PCA) are classic examples of using linear transformations for projecting data into more interpretable low dimensional spaces.  Unsupervised PCA seeks to identify a set of axes that best explain the variance contained in the data. LDA takes a supervised approach, minimizing the intra-class variance and maximizing the inter-class variance given class labeled data points.

Much of the recent work in Distance Metric Learning has focused on learning Mahalanobis distances on the basis of pairwise similarity/dissimilarity constraints. These methods have the same goals as LDA; pairs of points labeled ``similar" should be close to one another while pairs labeled ``dissimilar" should be distant. MMC \citep{xing2002distance}, a method for identifying a Mahalanobis metric for clustering with side information, uses semidefinite programming to identify a metric that maximizes the sum of distances between points labeled with different classes subject to the constraint that the sum of distances between all points with similar labels be less than some constant.  

Large Margin Nearest Neighbor (LMNN) \citep{weinberger2005distance} similarly uses semidefinite programming to identify a Mahalanobis distance.  In this setting, the algorithm minimizes the sum of distances between a given point and its similarly labeled neighbors while forcing differently labeled neighbors outside of its neighborhood.  This method has been shown to be computationally efficient \citep{weinberger2008fast} and, in contrast to the similarly motivated Neighborhood Component Analysis \citep{goldberger2004neighbourhood}, is guaranteed to converge to a globally optimal solution.  
Information Theoretic Metric Learning (ITML) \citep{davis2007information} is another popular Distance Metric Learning technique. ITML minimizes the Kullback-Liebler divergence between an initial guess of the matrix that parameterizes the Mahalanobis distance and a solution that satisfies a set of constraints.  
For surveys of the vast metric learning literature, see \citep{kulis2012metric,bellet2013survey,yang2006distance}.

In a dynamic environment, it is necessary to track the changing metric at different times, computing a sequence of estimates of the metric, and to be able to compute those estimates online. Online learning \cite{cesa2006prediction} meets these criteria by efficiently updating the estimate every time a new data point is obtained, instead of solving an objective function formed from the entire dataset. Many online learning methods have regret guarantees, that is, the loss in performance relative to a batch method is provably small \cite{cesa2006prediction,duchi2010composite}. In practice, however, the performance of an online learning method is strongly influenced by the learning rate, which may need to vary over time in a dynamic environment \cite{daniely2015strongly,mcmahan2010,duchi2010}. 

Adaptive online learning methods attempt to address the learning rate problem by continuously updating the learning rate as new observations become available. For learning static parameters, AdaGrad-style methods \cite{mcmahan2010,duchi2010} perform gradient descent steps with the step size adapted based on the magnitude of recent gradients. Follow the regularized leader (FTRL) type algorithms adapt the regularization to the observations \cite{mcmahan2014analysis}. Recently, a method called Strongly Adaptive Online Learning (SAOL) has been proposed for learning parameters undergoing $K$ discrete changes. SAOL maintains several learners with different learning rates and selects the best one based on recent performance \cite{daniely2015strongly}. Several of these adaptive methods have provable regret bounds \cite{mcmahan2014analysis,herbster1998tracking,hazan2007adaptive}. These typically guarantee low total regret (i.e. regret from time 0 to time $T$) at every time \cite{mcmahan2014analysis}. SAOL, on the other hand, attempts to have low \emph{static} regret on every subinterval, as well as low regret overall \cite{daniely2015strongly}. This allows tracking of discrete changes, but not slow drift. 


The remainder of this paper is structured as follows. In Section \ref{sec:problem} we formalize the distance metric tracking problem, and section \ref{Sec:COMIDLearn} presents the basic COMID online learner. 
Section \ref{Sec:SAOML} presents our COMID-SADL algorithm, a method of adaptively combining COMID learners with different learning rates. Strongly adaptive bounds on the dynamic regret are presented in Section \ref{Sec:Bounds}, and results on both synthetic data and a text review dataset are presented in Section \ref{sec:results}. Section \ref{sec:conclusion} concludes the paper.

%
%
%
%



\section{Problem Formulation} \label{sec:problem}

Metric learning seeks to learn a metric that encourages data points marked as similar to be close and data points marked as different to be far apart. The time-varying Mahalanobis distance at time $t$ is parameterized by $\mathbf M_t$ as
\begin{equation}
d_{M_t}^2(\mathbf{x},\mathbf{z}) = (\mathbf{x}-\mathbf{z})^T \mathbf M_t (\mathbf{x-z})
\end{equation}
where $\mathbf M_t \in \mathbb{R}^{n\times n} \succeq 0  $.

Suppose a temporal sequence of similarity constraints are given, where each constraint is the triplet $(\mathbf{x}_t,\mathbf z_t,y_t)$, $\mathbf x_t$ and $\mathbf z_t$ are data points in $\mathbb{R}^n$, and the label $y_t = +1$ if the points $\mathbf x_t, \mathbf z_t$ are similar at time $t$ and $y_t = -1$ if they are dissimilar. 

Following \cite{kunapuli2012mirror}, we introduce the following margin based constraints:
\begin{align}
\label{Eq:consts}
t | y_t = 1: \: d_{M_t}^2(\mathbf{x}_t,\mathbf{z}_t) \leq \mu-1;\qquad 
t | y_t = -1: \: d_{M_t}^2(\mathbf{x}_t,\mathbf{z}_t) \geq \mu +1,
\end{align}
where $\mu$ is a threshold that controls the margin between similar and dissimilar points. 
A diagram illustrating these constraints and their effect is shown in Figure \ref{Fig:Con}.
These constraints are softened by penalizing violation of the constraints with a convex loss function $\ell$.  This gives a loss function
\begin{align}
\label{Eq:Objective}
&\mathcal{L}(\{\mathbf M_t,\mu\}) =  \frac{1}{T} \sum_{t=1}^T\ell(y_t(\mu - \mathbf{u}_t^T \mathbf M_t \mathbf{u}_t)) + \rho r(\mathbf M_t) = \frac{1}{T} \sum_{t=1}^T f_t(\mathbf M_t,\mu) 
,
\end{align}
where $\mathbf{u}_t = \mathbf{x}_t - \mathbf{z}_t$, $r$ is the regularizer and $\rho$ the regularization parameter. Kunapuli and Shavlik \cite{kunapuli2012mirror} propose using nuclear norm regularization ($r(\mathbf M) = \|\mathbf M\|_*$) to encourage projection of the data onto a low dimensional subspace (feature selection/dimensionality reduction), and we have also had success with the elementwise L1 norm ($r(\mathbf M) = \|\mathrm{vec}(\mathbf M)\|_1$). In what follows, we develop an adaptive online method to minimize the loss subject to nonstationary smoothness constraints on the sequence of metric estimates $\mathbf M_t$.  
\begin{figure}[htb]
\centering
\includegraphics[width=3.0in]{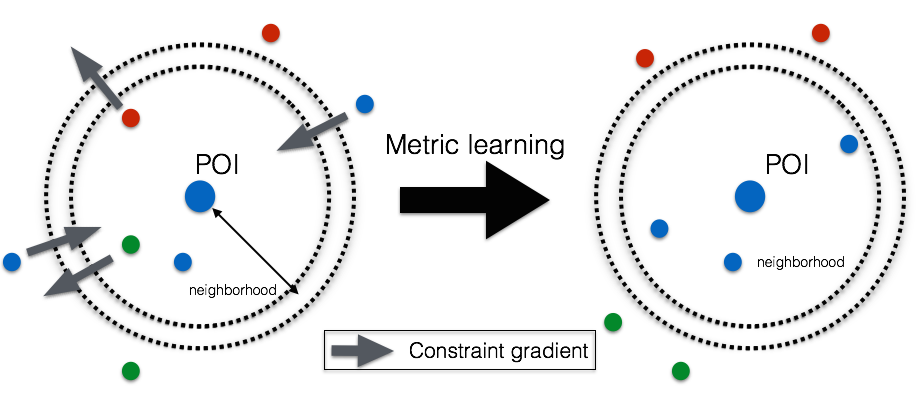}
\caption{Visualization of the margin based constraints \eqref{Eq:consts}, with colors indicating class. The goal of the metric learning constraints is to move target neighbors towards the point of interest (POI), while moving points from other classes away from the target neighborhood. }\label{Fig:Con}
\end{figure}
\section{Composite Objective Mirror Descent Update}
\label{Sec:COMIDLearn}
Viewing the acquisition of new data points as stochastic realizations of the underlying distribution \cite{kunapuli2012mirror} suggests the use of composite objective stochastic mirror descent techniques (COMID). 

For the loss \eqref{Eq:Objective} and learning rate $\eta_t$, COMID \cite{duchi2010composite} gives
\begin{align}
\label{Eq:COMID}
\hat{\mathbf M}_{t+1} =& \arg \min_{\mathbf M \succeq 0} B_\psi(\mathbf M,\hat{\mathbf M}_t) + \eta_t \langle \nabla_M \ell_t(\hat{\mathbf M}_t,\mu_t), \mathbf M-\hat{\mathbf M}_t\rangle + \eta_t \rho \|\mathbf M\|_*\\\nonumber
\hat{\mu}_{t+1} =& \arg \min_{\mu \geq 1} B_\psi(\mu,\hat{\mu}_t) + \eta_t \nabla_\mu \ell_t(\hat{\mathbf M}_t, \hat{\mu}_t)'(\mu - \hat{\mu}_t),
\end{align}
where $B_\psi$ is any Bregman divergence. 
In \citep{kunapuli2012mirror} a closed-form algorithm for solving the minimization in \eqref{Eq:COMID} with $r(\mathbf M) = \|\mathbf M\|_*$ is developed for a variety of common losses and Bregman divergences, involving rank one updates and eigenvalue shrinkage. 

The output of COMID depends strongly on the choice of $\eta_t$. Critically, the optimal learning rate $\eta_t$ depends on the rate of change of $\mathbf{M}_t$ \cite{hall2015online}, and thus will need to change with time to adapt to nonstationary drift. 
Choosing an optimal sequence for $\eta_t$ is clearly not practical in an online setting with nonstationary drift. We thus introduce COMID-Strongly Adaptive Dynamic Learning (COMID-SADL) as a method to adaptively choose an appropriate learning rate $\eta_t$.




\section{COMID-SADL} \label{sec:algorithm}
\label{Sec:SAOML}

Define a set $\mathcal{I}$ of intervals $I = [t_{I1}, t_{I2}]$ such that the lengths $|I|$ of the intervals are proportional to powers of two, i.e. $|I| = I_0 2^j$, $j = 0, \dots$, with an arrangement that is a dyadic partition of the temporal axis. The first interval of length $|I|$ starts at $t=|I|$ (see Figure \ref{Fig:SAOL}), and additional intervals of length $|I|$ exist such that the rest of time is covered. 

Every interval $I$ is associated with a base COMID learner that operates on that interval. Each learner \eqref{Eq:COMID} has a constant learning rate proportional to the inverse square of the length of the interval, i.e. $\eta_t(I) = \eta_0/\sqrt{|I|}$. Each learner (besides the coarsest) at level $j$ ($|I| = I_0 2^j$) is initialized to the last estimate of the next coarsest learner (level $j-1$) (see Figure \ref{Fig:Backdate}). This strategy is equivalent to ``backdating" the interval learners so as to ensure appropriate convergence has occurred before the interval of interest is reached, and is effectively a quantized square root decay of the learning rate. 


Thus, at a given time $t$, a set $\mathrm{ACTIVE}(t) \subseteq \mathcal{I}$ of $\mathrm{floor}(\log_2 t)$ intervals/COMID learners are active, running in parallel. Because the metric being learned is changing with time, learners designed for low regret at different scales will have different performance (analogous to the classical bias/variance tradeoff). In other words, there is a scale $I_{opt}$ optimal at a given time.

To select the appropriate scale, we compute weights $w_t(I)$ that are updated based on the learner's recent estimated regret. Our loss function in \eqref{Eq:Objective} is unbounded, however, it is a relaxation of an underlying 0-1 loss. For purposes of updating the weights, we propose using a nonlinearity to create a 0-1 loss with a smooth transitions scaled by a parameter $c$. We choose a linear transition as the nonlinearity
\begin{equation}
\label{Eq:log}
\ell_{t,c}(x_t | M_t,\mu_t)= \frac{1}{c} \min\{c, \ell_{t}(x_t | M_t,\mu_t)\}.
\end{equation}
We found that using a logistic nonlinearity also gave good results.
We set $c = 2$ in all our experiments. The weight update, inspired by the multiplicative weight (MW) literature, is given by  
\begin{align}
\label{eq:estreg}
r_t(I) =&\left(\sum_I\frac{w_t(I)}{\sum_{I} w_t(I)}\ell_{t,c} (\mathbf{M}_t(I),\mu_t(I))\right) - \ell_{t,c}(\mathbf{M}_t(I),\mu_t(I))\\\nonumber
w_{t+1}(I) =& w_{t}(I) (1 + \eta_I r_{t}(I)), \quad \forall t \in I.
\end{align}
These hold for all $I \in \mathcal{I}$, where $\eta_I = \min\{1/2, 1/\sqrt{|I|}\}$, $\mathbf{M}_t(I),\mu_t(I)$ are the outputs at time $t$ of the learner on interval $I$, and $r_t(I)$ is called the estimated regret of the learner on interval $I$ at time $t$. The initial value of $w(I)$ is $\eta_I$. Essentially, this is highly weighting low loss learners and lowly weighting high loss learners. 

\begin{figure}[htb]
\centering
\includegraphics[width=3.2in]{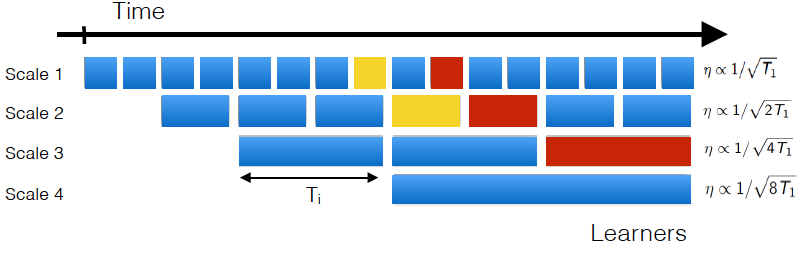}
\caption{COMID-SADL - Learners at multiple scales run in parallel. Recent observed losses for each learner are used to create weights used to select the appropriate scale at each time. Each yellow and red learner is initialized by the output of the previous learner of the same color, that is, the learner of the next shorter scale.} 
\label{Fig:Backdate}\label{Fig:SAOL}
\end{figure}



For any given time $t$, the output of the learner of interval $I \in \mathrm{ACTIVE}(t)$ is randomly selected as the output of COMID-SADL with probability
\begin{align}
\label{Eq:Select}
\mathrm{Pr}(\hat{I}_t = I) = \frac{w_t(I)}{\sum_{I \in \mathrm{ACTIVE}(t)} w_t(I)}. 
\end{align}
COMID-SADL is summarized in Algorithm 1.


\begin{algorithm}[htb]
\caption{COMID-SADL}\label{Alg:SAOML}
\begin{algorithmic}[1]
\STATE Initialize: $w_1(I)$
\FOR{$t = 1$ to $T$}
\STATE Initialize new learner if needed. New learner at scale $j> 0$: initialize to the last estimate of learner at scale $j-1$.
\STATE Choose $\hat{I}\in \mathrm{ACTIVE}(t)$ according to \eqref{Eq:Select}.
\STATE COMID update \eqref{Eq:COMID} for all active learners.
\STATE Set $\mathbf{M}_t \gets \mathbf{M}_t(\hat{I})$, $\mu_t \gets \mu_t(\hat{I})$
\STATE Obtain constraint $(\mathbf{x}_{t},\mathbf{z}_{t},y_{t})$, compute loss $\ell_{t,c}(\cdot)$.
\FOR{$I \in \mathrm{ACTIVE}(t)$}
\STATE Compute estimated regret $r_t(I)$ \eqref{eq:estreg} and update weights: 
$w_{t+1}(I) = w_{t}(I) (1 + \eta_I r_{t}(I)).$
\ENDFOR
\ENDFOR
\STATE Return $\{\mathbf{M}_t,\mu_t\}$.
\end{algorithmic}
\end{algorithm}

\section{Strongly Adaptive Dynamic Regret}
\label{Sec:Bounds}




The standard static regret is defined as
 \begin{equation}
R_{\mathcal{B},static}(I) = \sum_{t\in I} f_t (\hat{\theta}_t) - \min_{\theta \in \Theta} \sum_{t \in I} f_t(\theta).
\end{equation}
where $f_t(\theta_t)$ is a loss with parameter $\theta_t$.
Since in our case the optimal parameter value $\theta_t$ is changing, the static regret of an algorithm $\mathcal{B}$ on an interval $I$ is not useful.
Instead, let $\mathbf w = \{\theta_t\}_{t \in [0,T]}$ be an arbitrary sequence of parameters. Then, the \emph{dynamic regret} of an algorithm $\mathcal{B}$ relative to a comparator sequence $\mathbf w$ on the interval $I$ is defined as
\begin{equation}
R_{\mathcal{B},\mathbf w} (I)= \sum_{t\in I} f_t(\hat{\theta}_t) -\sum_{t\in I} f_t (\theta_t),
\end{equation}
where $\hat{\theta}_t$ are generated by $\mathcal{B}$. This allows for a dynamically changing estimate. 

In \cite{hall2015online} the authors derive dynamic regret bounds that hold over all possible sequences $\mathbf w$ such that $\sum_{t\in I} \|\theta_{t+1} - \theta_t\| \leq \gamma$, i.e. bounding the total amount of variation in the estimated parameter. Without this temporal regularization, minimizing the loss would cause $\theta_t$ to grossly overfit. In this sense, setting the comparator sequence $\mathbf w$ to the ``ground truth sequence" or ``batch optimal sequence" both provide meaningful intuitive bounds. 


Strongly adaptive regret bounds \cite{daniely2015strongly} have claimed that static regret is low on every subinterval, instead of only low in the aggregate. 
We use the notion of dynamic regret to introduce strongly adaptive dynamic regret bounds, proving that \emph{dynamic regret is low on every subinterval $I \subseteq [0,T]$ simultaneously}. 
In the supplementary material, we prove the following: 
\begin{theorem}[Strongly Adaptive Dynamic Regret]
\label{Thm:SADML}
Let $\mathbf w = \{\mathbf M_t\}_{t \in [0,T]}$ be any arbitrary sequence of metrics on the interval $[0,T]$, and define $\gamma_{\mathbf w}(I) = \sum_{t \in I} \|\mathbf M_{t+1} - \mathbf M_t\|$. 
Then COMID-SADL (Algorithm 1) satisfies
\begin{align}
\label{Eq:saRegretML}
E[R_{COMID-SADL,\mathbf w }(I) ] \leq \frac{4}{2^{1/2} - 1} C (1 + \gamma_{\mathbf{w}}(I) )  \sqrt{|I|} + 40 \log (s+1)\sqrt{|I|},
\end{align}
for every subinterval $I = [q,s] \subseteq [0, T]$ simultaneously. $C$ is a constant, and the expectation is with respect to the random output of the algorithm. 
\end{theorem}
In a dynamic setting, bounds of this type are particularly desirable because they allow for changing \emph{drift rate} and guarantee quick recovery from \emph{discrete changes}.
For instance, suppose $K$ discrete switches (large parameter changes or changes in drift rate) occur at times $t_i$ satisfying $0=t_0 < t_1< \dots< t_K=T$. Then since $\sum_{i = 1}^K \sqrt{|t_{i-1} - t_i|} \leq \sqrt{KT}$, this implies that the total expected dynamic regret on $[0,T]$ remains low ($O(\sqrt{KT})$), while simultaneously guaranteeing that an appropriate learning rate is used on each subinterval $[t_i, t_{i+1}]$. 

\section{Results} \label{sec:results}
\subsection{Synthetic Data}


We run our metric learning algorithms on a synthetic dataset undergoing different types of simulated metric drift. We create a synthetic 2000 point dataset with 2 independent 50-20-30\% clusterings (A and B) in disjoint 3-dimensional subspaces of $\mathbb{R}^{25}$. The clusterings are formed as 3-D Gaussian blobs, and the remaining 19-dimensional subspace is filled with iid Gaussian noise.


We create a scenario exhibiting nonstationary drift, combining continuous drifts and shifts between the two clusterings (A and B). To simulate continuous drift, at each time step we perform a small random rotation of the dataset. The drift profile is shown in \ref{Fig:None1}. For the first interval, partition A is used and the dataset is static, no drift occurs. Then, the partition is changed to B, followed by an interval of first moderate, then fast, and then moderate drift. Finally, the partition reverts back to A, followed by slow drift.

\begin{figure}[htb]
\centering
\includegraphics[width=5.3in]{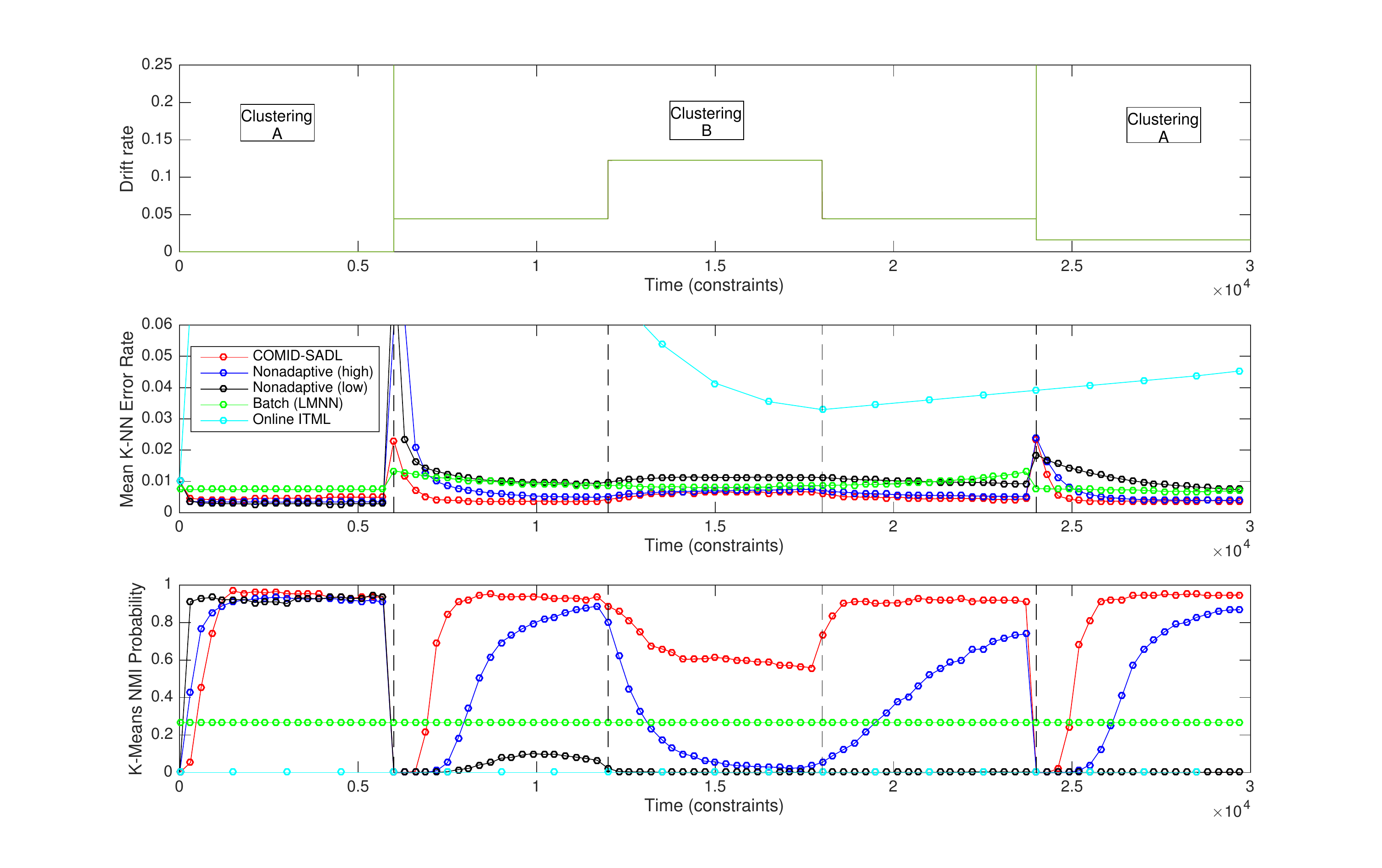}
\caption{
Tracking of a changing metric. Top: Rate of change (scaled Frobenius norm per tick) of the generating metric as a function of time. The large changes result from a change in clustering labels. Metric tracking performance is computed for COMID-SADL (adaptive), nonadaptive COMID (high learning rate), nonadaptive COMID (low learning rate), the batch solution (LMNN), and online ITML, and averaged over 3000 random trials. Shown as a function of time is the mean k-NN error rate (middle) and the probability that the k-means NMI exceeds $0.8$ (bottom). 
Note that COMID-SADL alone is able to effectively adapt to the variety of discrete changes and changes in drift rate.}
\label{Fig:None1}
\end{figure}
We generate a series of $T$ constraints from random pairs of points in the dataset, incorporating the simulated drift, running each experiment with 3000 random trials. For each experiment conducted in this section, we evaluate performance using two metrics. 
We plot the K-nearest neighbor error rate, using the learned embedding at each time point, averaging over all trials. We quantify the clustering performance by plotting the empirical probability that the normalized mutual information (NMI) of the K-means clustering of the unlabeled data points in the learned embedding at each time point exceeds 0.8 (out of a possible 1). We believe clustering NMI, rather than k-NN performance, is a more realistic indicator of metric learning performance, at least in the case where finding a relevant embedding is the primary goal.


In our results, we consider both COMID-SADL, nonadaptive COMID \citep{kunapuli2012mirror}, LMNN (batch) \citep{weinberger2005distance}, and online ITML \cite{davis2007information}. All parameters were set via cross validation. For nonadaptive COMID, we set the high learning rate using cross validation for moderate drift, and we set the low learning rate via cross validation in the case of no drift. The results are shown in Figure \ref{Fig:None1}. Online ITML fails due to its bias agains low-rank solutions \cite{davis2007information}, and the batch method and low learning rate COMID fail due to an inability to adapt. The high learning rate COMID does well at first, but as it is optimized for slow drift it cannot adapt to the changes in drift rate as well or recover quickly from the two partition changes. COMID-SADL, on the other hand, adapts well throughout the entire interval as expected.

\subsection{Clustering Product Reviews}


As an example real data task, we consider clustering Amazon text reviews, using the Multi-Domain Sentiment Dataset \citep{blitzer2007biographies}. We use the 11402 reviews from the Electronics and Books categories, and preprocess the data by computing word counts for each review and 2369 commonly occurring words, thus creating 11402 data points in $\mathbb{R}^{2369}$. Two possible clusterings of the reviews are considered: product category (books or electronics) and sentiment (positive: star rating 4/5 or greater, or negative: 2/5 or less).

Figures \ref{Fig:MLBooks} and \ref{Fig:MLStars} show the first two dimensions of the embeddings learned by static COMID for the category and sentiment clusterings respectively. Also shown are the 2-dimensional standard PCA embeddings, and the k-NN classification performance both before embedding and in each embeddings. As expected, metric learning is able to find embeddings with improved class separability. We emphasize that while improvements in k-NN classification are observed, we use k-NN merely as a way to quantify the separability of the classes in the learned embeddings. In these experiments, we set the regularizer $r(\cdot)$ to the elementwise L1 norm to encourage sparse features. 

\begin{figure}[htb]
\centering
\includegraphics[width=1.5in]{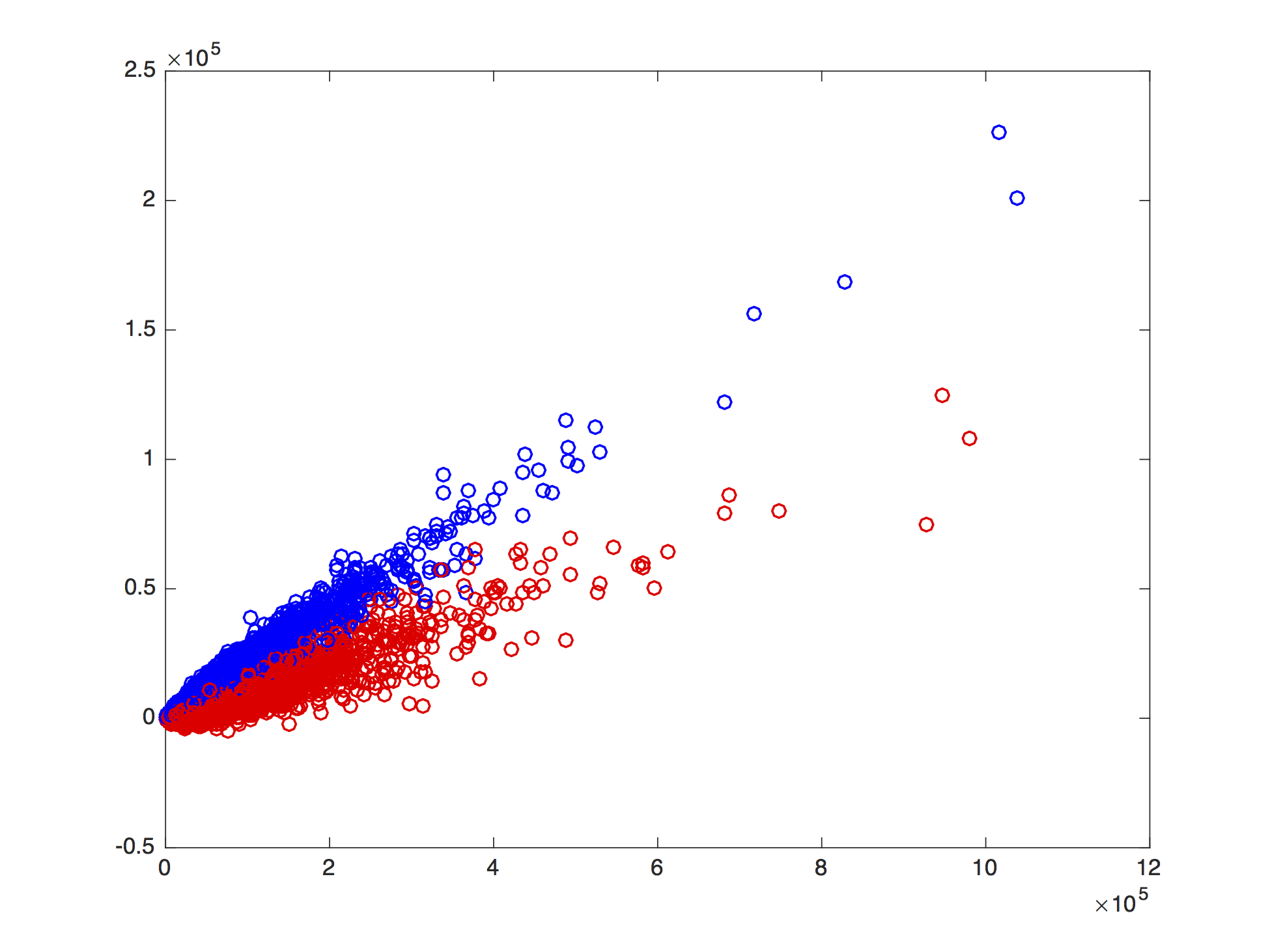}\includegraphics[width=1.5in]{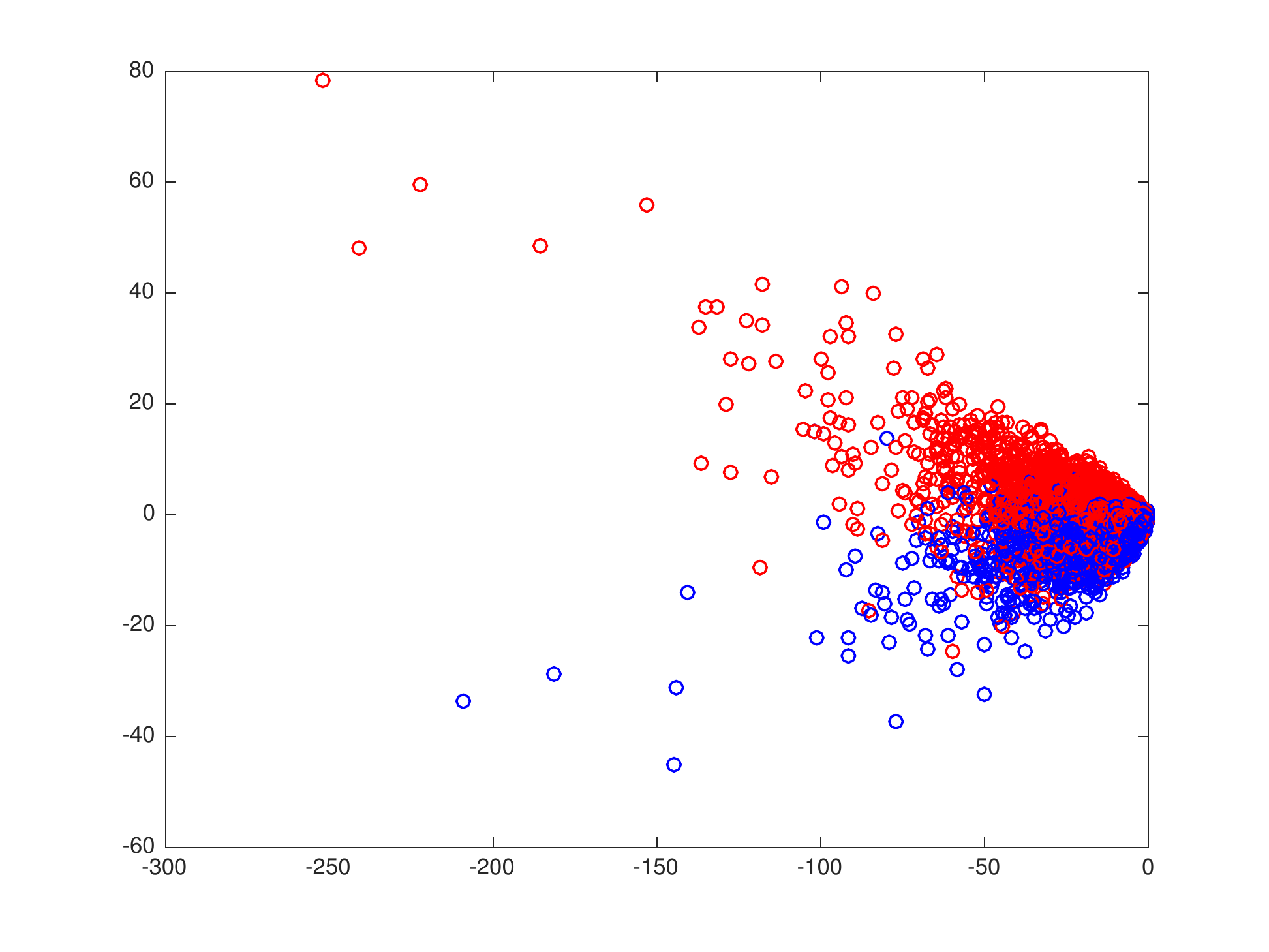}
\caption{Metric learning for product type clustering. Book reviews blue, electronics reviews red. Original LOO k-NN error rate 15.3\%. Left: First two dimensions of learned COMID-SADL embedding (LOO k-NN error rate 11.3\%). Right: embedding from PCA (k-NN error 20.4\%). Note improved separation of the clusters using COMID-SADL. }
\label{Fig:MLBooks}
\end{figure}

\begin{figure}[htb]
\centering
\includegraphics[width=1.5in]{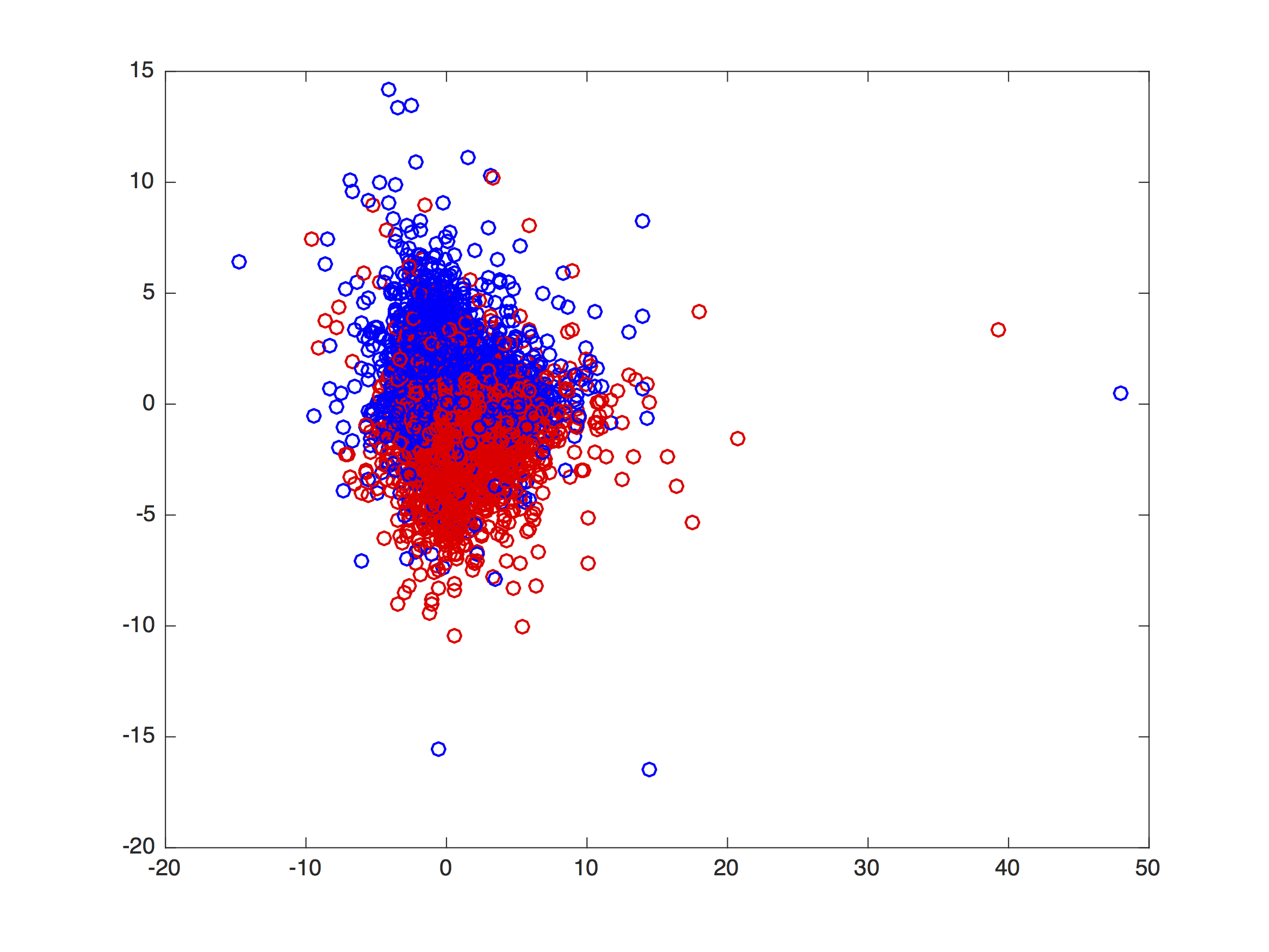}\includegraphics[width=1.5in]{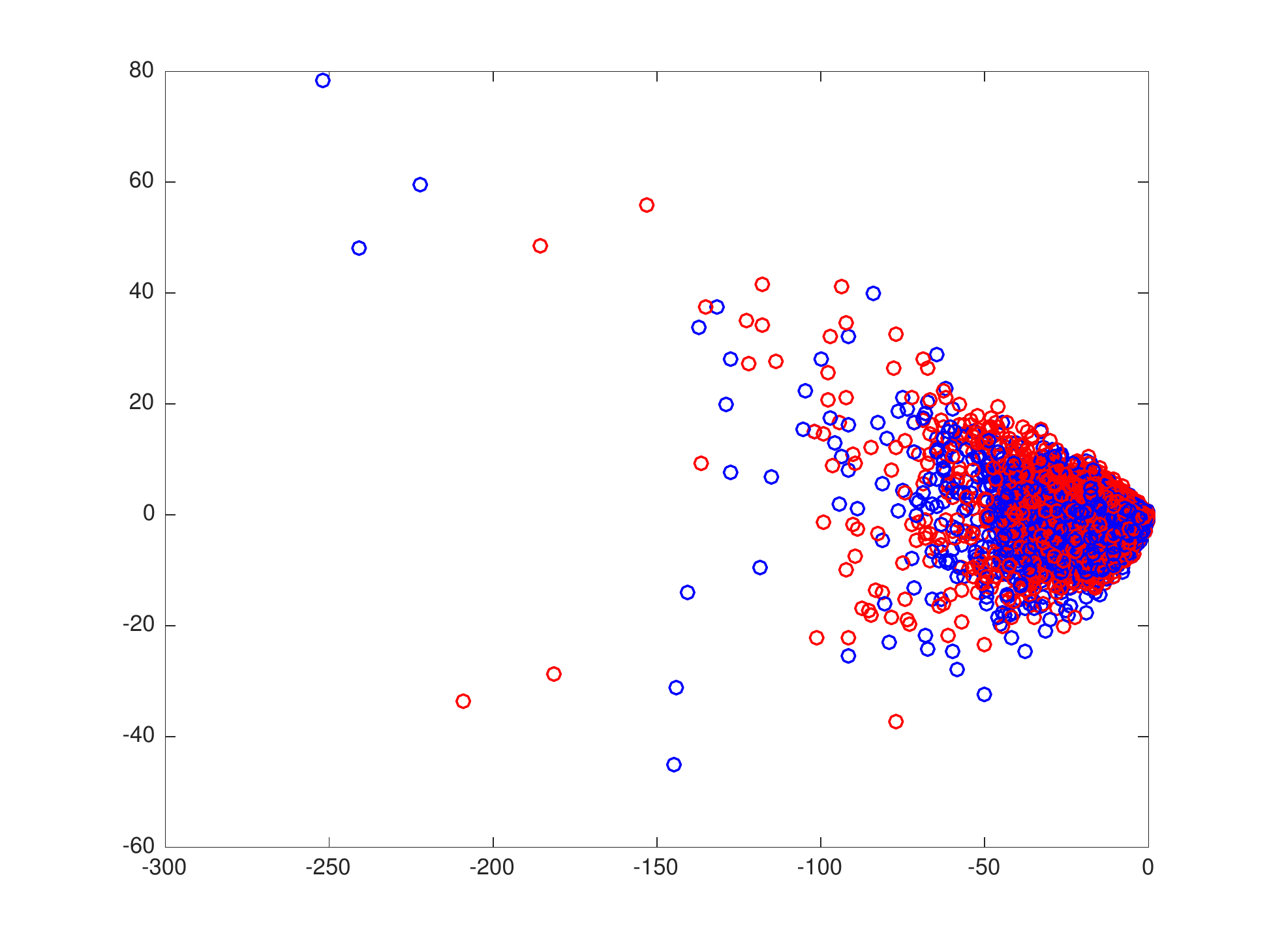}
\caption{Metric learning for sentiment clustering. Positive reviews blue, negative red. Original LOO k-NN error rate 35.7\%. Left: First two dimensions of learned COMID-SADL embedding (LOO k-NN error rate 23.5\%). Right: embedding from PCA (k-NN error 41.9\%). Note improved separation of the clusters using COMID-SADL. }
\label{Fig:MLStars}
\end{figure}

We then conducted drift experiments where the clustering changes. The change happens after the metric learner for the original clustering has converged, hence the nonadaptive learning rate is effectively zero. For each change, we show the k-NN error rate in the learned COMID-SADL embedding as it adapts to the new clustering. Emphasizing the visualization and computational advantages of a low-dimensional embedding, we computed the k-NN error after projecting the data into the first 5 dimensions of the embedding. Also shown are the results for a learner where an oracle allows reinitialization of the metric to the identity at time zero, and the nonadaptive learner for which the learning rate is not increased. Figure \ref{Fig:RealChangeBoth} (left) shows the results when the clustering changes from the four class sentiment + type partition to the two class product type only partition, and Figure \ref{Fig:RealChangeBoth} (right) shows the results when the partition changes from sentiment to product type. In the first case, the similar clustering allows COMID-SADL to significantly outperform even the reinitialized method, and in the second remain competitive where the clusterings are unrelated.




\begin{figure}[htb]
\centering
\includegraphics[width=1.7in]{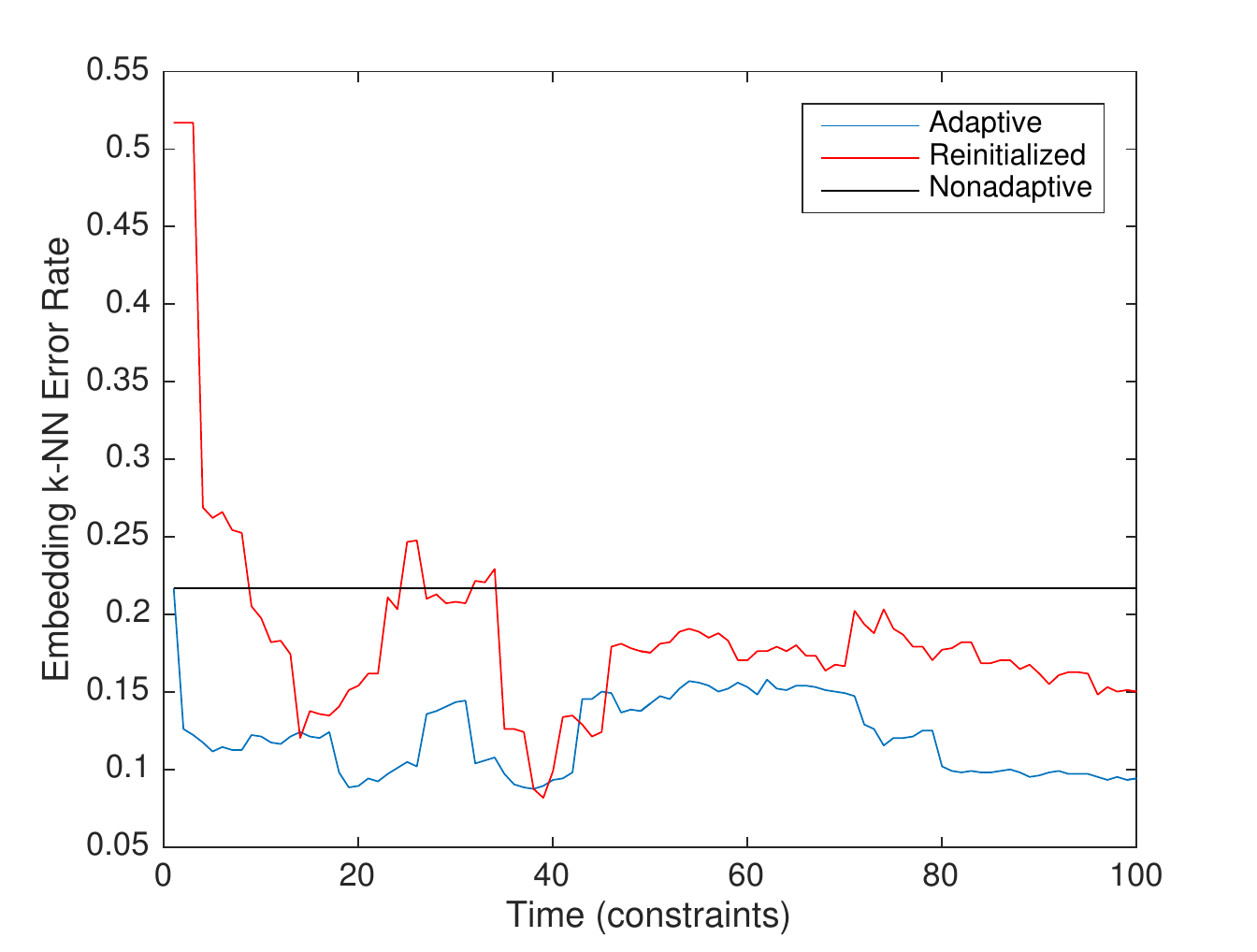}\includegraphics[width=1.7in]{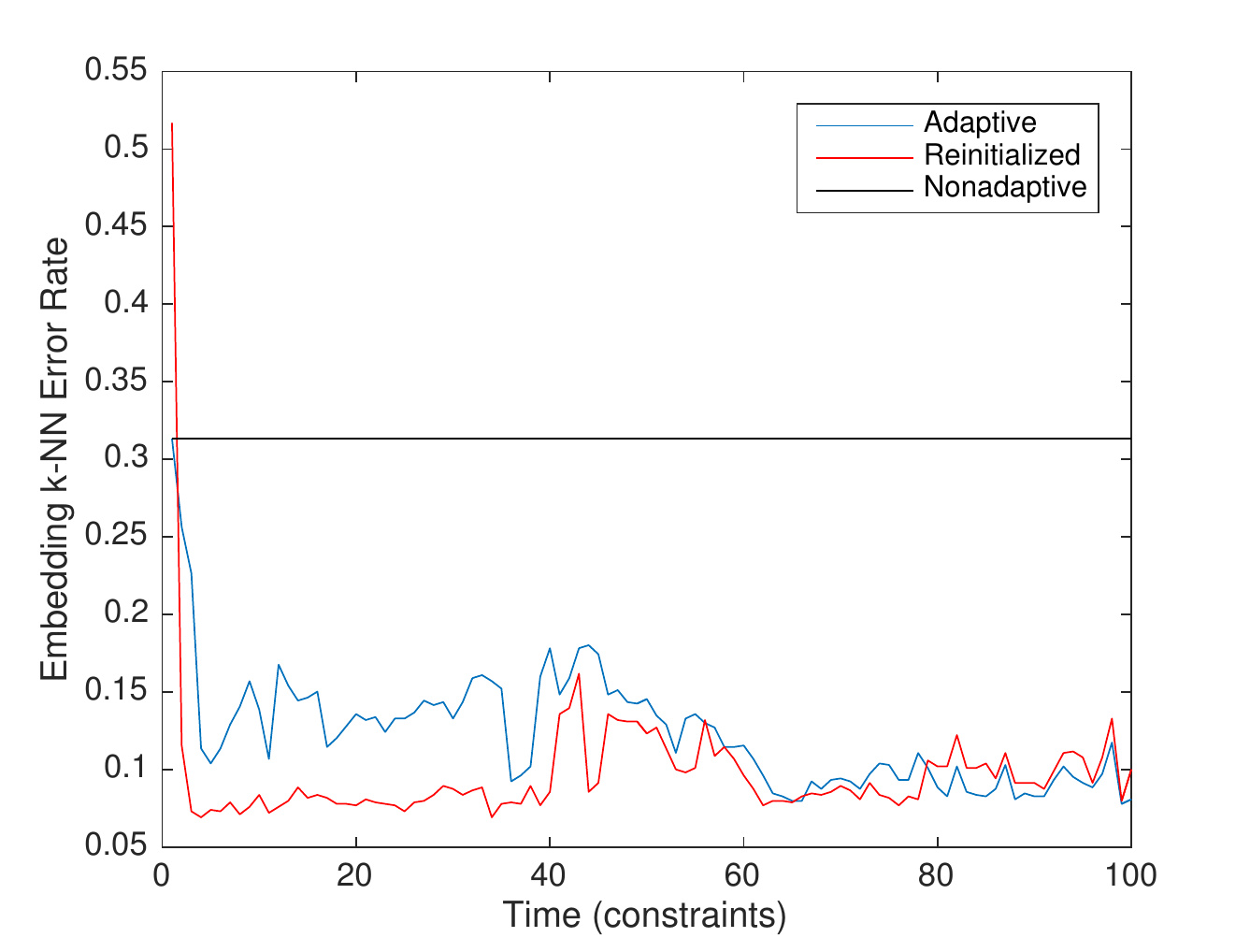}
\caption{Metric drift in Amazon review data. Left: Change from product type + sentiment clustering to simply product type; Right: Change from sentiment to product type clustering. The proposed COMID-SADL adapts to changes, tracking the clusters as they evolve. The oracle reinitialized mirror descent method (COMID) learner has higher tracking error and the nonadaptive learner (straight line) does not track the changes at all. }
\label{Fig:RealChangeBoth}
\end{figure}


\section{Conclusion and Future Work}\label{sec:conclusion}

Learning a metric on a complex dataset enables both unsupervised methods and/or a user to home in on the problem of interest while de-emphasizing extraneous information. When the problem of interest or the data distribution is nonstationary, however, the optimal metric can be time-varying. We considered the problem of tracking a nonstationary metric and presented an efficient, strongly adaptive online algorithm, called COMID-SADL, that has strong theoretical regret guarantees. Performance of our algorithm was evaluated both on synthetic and real datasets, demonstrating its ability to learn and adapt quickly in the presence of changes both in the clustering of interest and in the underlying data distribution. 

Potential directions for future work include the learning of more expressive metrics beyond the Mahalanobis metric, the incorporation of unlabeled data points in a semi-supervised learning framework \citep{bilenko2004integrating}, and the incorporation of an active learning framework to select which pairs of data points to obtain labels for at any given time \citep{settles2012active}.

\section{Acknowledgments}
The Lincoln Laboratory portion of this work was sponsored by the Assistant Secretary of Defense for Research and Engineering under Air Force Contract \#FA8721-05-C-0002. Opinions, interpretations, conclusions and recommendations are those of the author and are not necessarily endorsed by the United States Government.

\section{Strongly Adaptive Dynamic Regret}



We will prove the following general theorem giving strongly adaptive dynamic regret bounds.
\begin{theorem}
\label{Thm:SAOL}
Let $\mathbf{w} = \{\theta_1, \dots, \theta_T\}$ be an arbitrary sequence of parameters and define $\gamma_\mathbf{w}(I) = \sum_{q \leq t < s} \|\theta_{t+1} - \theta_t\|$ as a function of $\mathbf{w}$ and an interval $I = [q,s]$. Choose a set of learners $\mathcal{B}$ such that given an interval $I$ the learner $\mathcal{B}_I$ satisfies
\begin{equation}
\label{Eq:AlgCond}
R_{\mathcal{B}_I,\mathbf{w}}(T) \leq C (1 + \gamma_{\mathbf{w}}(I)) \sqrt{|I|}
\end{equation}
for some constant $C > 0$. Then the strongly adaptive dynamic learner ${SADL}^\mathcal{B}$ (COMID-SADL) using $\mathcal{B}$ as the interval learners satisfies
\begin{align}
\label{Eq:saRegret}
R_{SADL^{\mathcal{B}},\mathbf{w}}(I) \leq 8C (1 + \gamma(I))\sqrt{| I|} + 40 \log (s + 1) \sqrt{|I|}
\end{align}
on every interval $I = [q,s] \subseteq [0,T]$.
\end{theorem}



The proof techniques are similar to those found in \cite{daniely2015strongly,blum2005external} which are in turn similar to the analysis of the Multiplicative Weights Update (MW) method. 

Define
\begin{align}
\tilde{w}_t =& \left\{\begin{array}{ll} 0 & t < q\\ 1 & t = q \\ \tilde{w}_{t-1}(I)(1+\eta_Ir_{t-1}(I)) & q < t \leq s + 1\\\tilde{w}_s(I) & t > s+1\end{array} \right. \\\nonumber
\tilde{W}_t =& \sum_{I \in \mathcal{I}} \tilde{w}_{t+1}(I).
\end{align}
Note that $w_t(I) = \eta_I I(t)\tilde{w}_t(I)$ where $I(t)$ is the indicator function for $I$.

We first prove a pair of lemmas.
\begin{lemma}
\label{lem1}
\[
E[ \tilde{W}_t] \leq t(\log(t) + 1)
\]
for all $t \geq 1$. 
\end{lemma}
\begin{proof}
For all $t \geq 1$, $|\{[q,s] \in \mathcal{I}: q = t\}| \leq \lfloor \log(t)\rfloor + 1$. Thus
\begin{align*}
\tilde{W}_{t+1} =& \sum_{I = [q,s] \in \mathcal{I}} \tilde{w}_{t+1}(I)\\
=& \sum_{I = [t+1,s] \in \mathcal{I}} \tilde{w}_{t+1}(I) + \sum_{I=[q,s]\in \mathcal{I}: q \leq t} \tilde{w}_{t+1}(I)\\
\leq& \log(t+1) + 1 + \sum_{I = [q,s] \in \mathcal{I}:q \leq t} \tilde{w}_{t+1}(I).
\end{align*}
Then
\begin{align*}
 \sum_{I = [q,s] \in \mathcal{I}:q \leq t} \tilde{w}_{t+1}(I) =& \sum_{I = [q,s] \in \mathcal{I}:q \leq t} \tilde{w}_{t}(I)(1 + \eta_I I(t) r_t(I) )\\
=& \tilde{W}_t + \sum_{I \in \mathcal{I}} w_t(I) r_t(I).
\end{align*}
Suppose that $E[\tilde{W}_t] \leq t(\log(t) + 1)$. Furthermore, note that
\begin{align*}
E\left[\sum_{I \in \mathcal{I}} w_t(I) r_t(I)\right] &= W_t \sum_{I \in \mathcal{I}} p_t(I)(E [\ell_t(x_t)] - \ell_t(x_t(I)))\\
&= E[\ell_t(x_t)] - \sum_{I \in \mathcal{I}} p_t(I)\ell_t(x_t(I)))\\
&=E[\ell_t(x_t)] - E[\ell_t(x_t)]\\
&= 0.
\end{align*}
since $x_t = x_t(I)$ with probability $p_t(I)$.
Thus
\begin{align*}
E[\tilde{W}_{t+1}] &\leq t(\log(t) + 1) + \log(t+1) + 1 + E\left[\sum_{I \in \mathcal{I}} w_t(I) r_t(I)\right]\\
&\leq (t+1) (\log(t+1) + 1).
\end{align*}
Since $E[\tilde{W}_1] = \tilde{W}_1 = \tilde{w}([1,1]) = 1$, the lemma follows by induction.

\end{proof}

\begin{lemma}
\label{lem2}
\[
E\sum_{t=q}^s r_t(I) \leq 5 \log(s+1)\sqrt{|I|},
\]
for every $I = [q,s] \in \mathcal{I}$.
\end{lemma}
\begin{proof}
Fix $I = [q,s] \in \mathcal{I}$. Recall that
\[
\tilde{w}_{s+1}(I) = \prod_{t=q}^s (1 + \eta_I I(t) r_t(I)) = \prod_{t=q}^s ( 1+ \eta_I r_t(I)).
\]
Since $\eta_I \in (0,1/2)$ and $\log(1 + x) \geq (x - x^2)$ for all $x \geq -1/2$,
\begin{align}
\label{eq:reglem2}
\log(\tilde{w}_{s+1}(I)) &= \sum_{t=q}^s \log(1 + \eta_I r_t(I))\\\nonumber
&\geq \sum_{t=q}^s \eta_I r_t(I) - \sum_{t=q}^s (\eta_I r_t(I))^2\\\nonumber
&\geq \eta_I\left( \sum_{t=q}^s r_t(I) - \eta_I |I|\right).
\end{align}
By Lemma \ref{lem1} we have
\[
E[\tilde{w}_{s+1}(I)] \leq E[\tilde{W}_{s+1}] \leq (s+1)(\log(s+1) + 1),
\]
so
\[
E[\log(\tilde{w}_{s+1}(I))] \leq \log(E[ \tilde{w}_{s+1}]) \leq \log(s+1) + \log(\log(s+1)+1).
\]
Combining with the expectation of \eqref{eq:reglem2} and dividing by $\eta_I$,
\begin{align*}
E\left[ \sum_{t=q}^s r_t (I)\right] &\leq \eta_I | I | + \frac{1}{\eta_I} (\log(s + 1) + \log(\log(s+1) + 1))\\
&\leq \eta_I | I | + 2 \eta_I^{-1} \log(s+ 1)\\
& = 5 \log(s+1) \sqrt{|I|}.
\end{align*}
since $x \geq \log(1 + x)$ and $\eta_I = \min\{1/2,|I|^{-1/2}\}$.

\end{proof}

Define the restriction of $\mathcal{I}$ to an interval $J \subseteq \mathbb{N}$ as $\mathcal{I}|_{J} = \{I \in \mathcal{I}: I \subseteq J\}$. Note the following lemma from \cite{daniely2015strongly}:
\begin{lemma}
\label{lem5}
Consider the arbitrary interval $I = [q,s] \subseteq \mathbb{N}$. Then, the interval $I$ can be partitioned into two finite sequences of disjoint and consecutive intervals, given by $(I_{-k}, \dots, I_0) \subseteq \mathcal{I}|_I$ and $(I_1, I_2, \dots, I_p) \subseteq \mathcal{I}|_I$, such that
\begin{align*}
\begin{array}{ll}
|I_{-i}|/|I_{-i+1}| \leq 1/2, & \forall i \geq 1,\\
|I_i|/|I_{i-1}| \leq 1/2, & \forall i \geq 2.
\end{array}
\end{align*} 
\end{lemma}

This enables us to extend the bounds to every arbitrary interval $I = [q,s] \subseteq [T]$ and thus complete the proof.

Let $I ={ \bigcup}_{i=-k}^p I_i$ be the partition described in Lemma \ref{lem5}. Then
\begin{align}
\label{eq:regThm}
R_{SADL^{\mathcal{B}},\mathbf{w}}(I) \leq \sum_{i \leq 0} R_{SADL^{\mathcal{B}},\mathbf{w}}(I_i) + \sum_{i \geq 1} R_{SADL^{\mathcal{B}},\mathbf{w}}(I_i).
\end{align}
By Lemma \ref{lem2} and \eqref{Eq:AlgCond},
\begin{align*}
\sum_{i \leq 0} R_{SADL^{\mathcal{B}},\mathbf{w}}(I_i) &\leq C \sum_{i\leq 0} (1 + \gamma_{\mathbf{w}}(I_i))\sqrt{|I_i|} + 5 \sum_{i \leq 0} \log(s_i + 1) \sqrt{I_i}\\
&\leq (C (1 + \gamma(I)) + 5\log(s_i + 1) )\sum_{i \leq 0}  \sqrt{I_i},
\end{align*}
since $\gamma_{\mathbf{w}}(I_i) \leq \gamma_{\mathbf{w}}(I)$ by definition.
By Lemma \ref{lem5},
\begin{align*}
\sum_{i \leq 0} \sqrt{|I_i|} \leq \frac{\sqrt{2}}{\sqrt{2} - 1}\sqrt{ |I|} \leq 4\sqrt{|I|}.
\end{align*}
This bounds the first term of the right hand side of Equation \eqref{eq:regThm}. The bound for the second term can be found in the same way. Thus, 
\[
R_{SADL^{\mathcal{B}},\mathbf{w}}(I) \leq (8C (1 + \gamma(I))\sqrt{| I|} + 40 \log (s + 1) \sqrt{|I|}.
\]
Since this holds for all $I$, this completes the proof.

\section{Online DML Dynamic Regret}

In this section, we derive the dynamic regret of the COMID metric learning algorithm. Recall that the COMID algorithm is given by
\begin{align}
\label{Eq:COMID}
\hat{\mathbf M}_{t+1} =& \arg \min_{\mathbf M \succeq 0} B_\psi(\mathbf M,\hat{\mathbf M}_t) \\\nonumber &+ \eta_t \langle \nabla_M \ell_t(\hat{\mathbf M}_t,\mu_t), \mathbf M-\hat{\mathbf M}_t\rangle + \eta_t \rho \|\mathbf M\|_*\\\nonumber
\hat{\mu}_{t+1} =& \arg \min_{\mu \geq 1} B_\psi(\mu,\hat{\mu}_t) + \eta_t \nabla_\mu \ell_t(\hat{\mathbf M}_t, \hat{\mu}_t)'(\mu - \hat{\mu}_t),
\end{align}
where $B_\psi$ is any Bregman divergence and $\eta_t$ is the learning rate parameter. From \cite{hall2015online} we have:


\begin{theorem}
\begin{align*}
G_\ell &= \max_{\theta \in \Theta,\ell \in \mathcal{L}} \|\nabla f(\theta) \| \\
\phi_{max} &= \frac{1}{2} \max_{\theta \in \Theta} \| \nabla \psi (\theta)\| \\
D_{max} &= \max_{\theta,\theta' \in \Theta} B_\psi(\theta' \| \theta) 
\end{align*}

Let the sequence $\hat{\mathbf{\theta}}_t = [\hat{\mathbf{M}}_t, \hat{\mu}_t]$, $t = 1,\cdots, T$ be generated via the COMID algorithm, and let $\mathbf w$ be an arbitrary sequence in $\mathcal{W}=\{\mathbf w | \sum_{t = 0}^{T-1} \|\theta_{t+1} - \theta_t\| \leq \gamma \}$. Then using $\eta_{t+1} \leq \eta_t$ gives a dynamic regret
\begin{equation}
R_\mathbf{w}([0, T]) \leq \frac{D_{max}}{\eta_{T+1}} + \frac{4\phi_{max}}{\eta_T} \gamma + \frac{G_\ell^2}{2\sigma} \sum_{t=1}^T \eta_t
\end{equation}

\end{theorem}
Using a nonincreasing learning rate $\eta_t$, we can then prove a bound on the dynamic regret for a quite general set of stochastic optimization problems. 




Applying this to our problem, we have 
%
\begin{align*}
G_\ell &= \max_{\|\mathbf{M}\| \leq c,t,\mu} \|\nabla (\ell_t(\mathbf{M},\mu) + \rho \|\mathbf{M}\|_*) \|_2 \\
\phi_{max} &= \frac{1}{2} \max_{\|\mathbf{M}\| \leq c} \| \nabla \psi (\mathbf{M})\|_2 \\
D_{max} &= \max_{\|\mathbf{M}\|,\|\mathbf{M}'\| \leq c} B_\psi(\mathbf{M}' \| \mathbf{M}) 
\end{align*}
For $\ell_t(\cdot)$ being the hinge loss and $\psi= \|\cdot\|_F^2$,
\begin{align*}
G_\ell &\leq \sqrt{ (\max_t \|\mathbf x_t-\mathbf z_t\|_2^2 + \rho)^2}\\
\phi_{max} &=  c\sqrt{n} \\
D_{max} &= 2c\sqrt{n}.
\end{align*}
The other two quantities are guaranteed to exist and depend on the choice of Bregman divergence and $c$. Thus,
\begin{corollary}[Dynamic Regret: ML COMID]
\label{Cor:DynReg}

Let the sequence $\hat{\mathbf{M}}_t, \hat{\mu}_t$ be generated by \eqref{Eq:COMID}, and let $\mathbf{w} = \{\mathbf{M}_t\}_{t=1}^T$ be an arbitrary sequence with $\|\mathbf{M}_t\| \leq c$. Then using $\eta_{t+1} \leq \eta_t$ gives
\begin{equation}
R_{\mathbf{w}} \leq \frac{D_{max}}{\eta_{T+1}} + \frac{4\phi_{max}}{\eta_T} \gamma + \frac{G_\ell^2}{2\sigma} \sum_{t=1}^T \eta_t
\end{equation}
and setting $\eta_t = \eta_0/\sqrt{T}$,
\begin{align}
R_{\mathbf{w}}([0, T]) \leq& \sqrt{T}\left(\frac{D_{max} + 4 \phi_{max} ( \sum_t \|\mathbf{M}_{t+1} - \mathbf{M}_t\|_F)}{\eta_0}+ \frac{\eta_0 G_\ell^2}{2\sigma}\right)\nonumber\\ \label{Eq:DynBound}
=& O\left(\sqrt{T} \left[ 1 + \sum_{t  = 1}^T \|\mathbf{M}_{t+1} - \mathbf{M}_t\|_F\right]\right).
\end{align}

\end{corollary}


Corollary \ref{Cor:DynReg} is a bound on the regret relative to the batch estimate of ${\mathbf{M}}_t$ that minimizes the total batch loss subject to a bounded variation $\sum_t \|\mathbf{M}_{t+1} - \mathbf{M}_t\|_F$. Also note that setting $\eta_t = \eta_0/\sqrt{t}$ gives the same bound as \eqref{Eq:DynBound}.

In other words, we pay a linear penalty on the total amount of variation in the underlying parameter sequence. From \eqref{Eq:DynBound}, it can be seen that the bound-minimizing $\eta_0$ increases with increasing $\sum_t \|\mathbf{M}_{t+1} - \mathbf{M}_t\|_F$, indicating the need for an adaptive learning rate.

For comparison, if the metric is in fact static then by standard stochastic mirror descent results \cite{hall2015online}
\begin{theorem}[Static Regret]
If $\hat{\mathbf{M}}_1 = 0$ and $\eta_t = (2\sigma D_{max})^{1/2}/(G_f \sqrt{T})$, then
\begin{equation}
R_{static}([0,T]) \leq G_f (2T D_{max}/\sigma)^{1/2}.
\end{equation}
\end{theorem}

%
%



\section{COMID-SADL Bound}






Let $\mathcal{B}^i$ be a COMID learner at any of the scales used in $SADL^\mathcal{B}$, with output $x_t(i)$. Define the relative regret
\[
\tilde{R}_{SADL,\mathbf{w}}^i(I) = \sum_{t \in I } \ell_{c}(x_t) - \ell_{c} (x_t(i))
\]
as the extra $\ell_c$ loss suffered relative to the algorithm $\mathcal{B}^i$. From the proof of Theorem 1 we have
\begin{lemma}

For any $c$, the following holds simultaneously for all $\mathcal{B}^i$ and $I$.
\[
\tilde{R}_{SADL,\mathbf{w}}^i(I) \leq 40 \log (s+1)|I|^{1/2}.
\]
\end{lemma}
This implies that SADL incurs at most $O(\sqrt{|I|})$ additional scaled 0-1 loss on any interval relative to each of the base learners, all of which have low regret in the convex loss. Due to the nonconvexity of the scaled 0-1 loss, it is difficult to state more for arbitrary $c$.

However, since $\|\mathbf{M}\| \leq c'$, $\ell_t(\mathbf{M}_t,\mu_t) \leq k =  \ell(c' \max_{t} \|\mathbf{x}_t - \mathbf{z}_t\|_2^2)$. Hence for $c = k$, $\ell_{c}(\cdot) = \frac{1}{c}\ell(\cdot)$ everywhere. Thus Corollary \ref{Cor:DynReg} can be used in Theorem \ref{Thm:SAOL}, giving 
\begin{theorem}[COMID-SADL]
\label{Thm:SADML}
Let $\mathcal{B}$ be the COMID algorithm of \eqref{Eq:COMID} with $\eta_t (I) = \eta_0/\sqrt{|I|}$. Then there exists a $c$ such that the strongly adaptive online learner $SADL^\mathcal{B}$ (COMID-SADL) satisfies
\begin{align}
\label{Eq:saRegretML}
R_{SADL,\mathbf{w}}(I) \leq 8 C (1 + \gamma_{\mathbf{w}}(I) )  |I|^{1/2} + 40 \log (s+1)|I|^{1/2}
\end{align}
for some constant $C$ and every interval $I = [q,s]$. 
\end{theorem}
Note that this bound also holds for the original convex loss $\ell$.






%

\bibliographystyle{plainnat}
\bibliography{metric_learning}

\begin{thebibliography}{3}
\providecommand{\natexlab}[1]{#1}
\providecommand{\url}[1]{\texttt{#1}}
\expandafter\ifx\csname urlstyle\endcsname\relax
  \providecommand{\doi}[1]{doi: #1}\else
  \providecommand{\doi}{doi: \begingroup \urlstyle{rm}\Url}\fi

\bibitem[Blum and Mansour(2005)]{blum2005external}
Avrim Blum and Yishay Mansour.
\newblock From external to internal regret.
\newblock In \emph{Learning theory}, pages 621--636. Springer, 2005.

\bibitem[Daniely et~al.(2015)Daniely, Gonen, and
  Shalev-Shwartz]{daniely2015strongly}
Amit Daniely, Alon Gonen, and Shai Shalev-Shwartz.
\newblock Strongly adaptive online learning.
\newblock \emph{ICML}, 2015.

\bibitem[Hall and Willett(2015)]{hall2015online}
E.C. Hall and R.M. Willett.
\newblock Online convex optimization in dynamic environments.
\newblock \emph{Selected Topics in Signal Processing, IEEE Journal of},
  9\penalty0 (4):\penalty0 647--662, June 2015.

\end{thebibliography}


\begin{thebibliography}{24}
\providecommand{\natexlab}[1]{#1}
\providecommand{\url}[1]{\texttt{#1}}
\expandafter\ifx\csname urlstyle\endcsname\relax
  \providecommand{\doi}[1]{doi: #1}\else
  \providecommand{\doi}{doi: \begingroup \urlstyle{rm}\Url}\fi

\bibitem[Bellet et~al.(2013)Bellet, Habrard, and Sebban]{bellet2013survey}
Aur{\'e}lien Bellet, Amaury Habrard, and Marc Sebban.
\newblock A survey on metric learning for feature vectors and structured data.
\newblock \emph{arXiv preprint arXiv:1306.6709}, 2013.

\bibitem[Bilenko et~al.(2004)Bilenko, Basu, and Mooney]{bilenko2004integrating}
Mikhail Bilenko, Sugato Basu, and Raymond~J Mooney.
\newblock Integrating constraints and metric learning in semi-supervised
  clustering.
\newblock In \emph{ICML}, page~11, 2004.

\bibitem[Bishop(2006)]{bishop2006pattern}
Christopher~M Bishop.
\newblock \emph{Pattern Recognition and Machine Learning}.
\newblock Springer, 2006.

\bibitem[Blitzer et~al.(2007)Blitzer, Dredze, Pereira,
  et~al.]{blitzer2007biographies}
John Blitzer, Mark Dredze, Fernando Pereira, et~al.
\newblock Biographies, bollywood, boom-boxes and blenders: Domain adaptation
  for sentiment classification.
\newblock In \emph{ACL}, volume~7, pages 440--447, 2007.

\bibitem[Cesa-Bianchi and Lugosi(2006)]{cesa2006prediction}
Nicolo Cesa-Bianchi and G{\'a}bor Lugosi.
\newblock \emph{Prediction, learning, and games}.
\newblock Cambridge University Press, 2006.

\bibitem[Daniely et~al.(2015)Daniely, Gonen, and
  Shalev-Shwartz]{daniely2015strongly}
Amit Daniely, Alon Gonen, and Shai Shalev-Shwartz.
\newblock Strongly adaptive online learning.
\newblock \emph{ICML}, 2015.

\bibitem[Davis et~al.(2007)Davis, Kulis, Jain, Sra, and
  Dhillon]{davis2007information}
Jason~V Davis, Brian Kulis, Prateek Jain, Suvrit Sra, and Inderjit~S Dhillon.
\newblock Information-theoretic metric learning.
\newblock In \emph{ICML}, pages 209--216, 2007.

\bibitem[Duchi et~al.(2010{\natexlab{a}})Duchi, Hazan, and Singer]{duchi2010}
John~C Duchi, Elad Hazan, and Yoram Singer.
\newblock Adaptive subgradient methods for online learning and stochastic
  optimization.
\newblock In \emph{COLT}, 2010{\natexlab{a}}.

\bibitem[Duchi et~al.(2010{\natexlab{b}})Duchi, Shalev-Shwartz, Singer, and
  Tewari]{duchi2010composite}
John~C Duchi, Shai Shalev-Shwartz, Yoram Singer, and Ambuj Tewari.
\newblock Composite objective mirror descent.
\newblock In \emph{COLT}, pages 14--26. Citeseer, 2010{\natexlab{b}}.

\bibitem[Goldberger et~al.(2004)Goldberger, Hinton, Roweis, and
  Salakhutdinov]{goldberger2004neighbourhood}
Jacob Goldberger, Geoffrey~E Hinton, Sam~T Roweis, and Ruslan Salakhutdinov.
\newblock Neighbourhood components analysis.
\newblock In \emph{Advances in neural information processing systems}, pages
  513--520, 2004.

\bibitem[Hall and Willett(2015)]{hall2015online}
E.C. Hall and R.M. Willett.
\newblock Online convex optimization in dynamic environments.
\newblock \emph{Selected Topics in Signal Processing, IEEE Journal of},
  9\penalty0 (4):\penalty0 647--662, June 2015.

\bibitem[Hastie et~al.(2005)Hastie, Tibshirani, Friedman, and
  Franklin]{hastie2005elements}
Trevor Hastie, Robert Tibshirani, Jerome Friedman, and James Franklin.
\newblock The elements of statistical learning: data mining, inference and
  prediction.
\newblock \emph{The Mathematical Intelligencer}, 27\penalty0 (2):\penalty0
  83--85, 2005.

\bibitem[Hazan and Seshadhri(2007)]{hazan2007adaptive}
Elad Hazan and C~Seshadhri.
\newblock Adaptive algorithms for online decision problems.
\newblock In \emph{Electronic Colloquium on Computational Complexity (ECCC)},
  volume~14, 2007.

\bibitem[Herbster and Warmuth(1998)]{herbster1998tracking}
Mark Herbster and Manfred~K Warmuth.
\newblock Tracking the best expert.
\newblock \emph{Machine Learning}, 32\penalty0 (2):\penalty0 151--178, 1998.

\bibitem[Kulis(2012)]{kulis2012metric}
Brian Kulis.
\newblock Metric learning: A survey.
\newblock \emph{Foundations and Trends in Machine Learning}, 5\penalty0
  (4):\penalty0 287--364, 2012.

\bibitem[Kunapuli and Shavlik(2012)]{kunapuli2012mirror}
Gautam Kunapuli and Jude Shavlik.
\newblock Mirror descent for metric learning: a unified approach.
\newblock In \emph{Machine Learning and Knowledge Discovery in Databases},
  pages 859--874. Springer, 2012.

\bibitem[Lee and Verleysen(2007)]{lee2007nonlinear}
John~A Lee and Michel Verleysen.
\newblock \emph{Nonlinear dimensionality reduction}.
\newblock Springer Science \& Business Media, 2007.

\bibitem[McMahan(2014)]{mcmahan2014analysis}
H~Brendan McMahan.
\newblock Analysis techniques for adaptive online learning.
\newblock \emph{arXiv preprint arXiv:1403.3465}, 2014.

\bibitem[McMahan and Streeter(2010)]{mcmahan2010}
H~Brendan McMahan and Matthew Streeter.
\newblock Adaptive bound optimization for online convex optimization.
\newblock In \emph{COLT}, 2010.

\bibitem[Settles(2012)]{settles2012active}
Burr Settles.
\newblock Active learning.
\newblock \emph{Synthesis Lectures on Artificial Intelligence and Machine
  Learning}, 6\penalty0 (1):\penalty0 1--114, 2012.

\bibitem[Weinberger and Saul(2008)]{weinberger2008fast}
Kilian~Q Weinberger and Lawrence~K Saul.
\newblock Fast solvers and efficient implementations for distance metric
  learning.
\newblock In \emph{ICML}, pages 1160--1167, 2008.

\bibitem[Weinberger et~al.(2005)Weinberger, Blitzer, and
  Saul]{weinberger2005distance}
Kilian~Q Weinberger, John Blitzer, and Lawrence~K Saul.
\newblock Distance metric learning for large margin nearest neighbor
  classification.
\newblock In \emph{Advances in Neural Information Processing System}, pages
  1473--1480, 2005.

\bibitem[Xing et~al.(2002)Xing, Jordan, Russell, and Ng]{xing2002distance}
Eric~P Xing, Michael~I Jordan, Stuart Russell, and Andrew~Y Ng.
\newblock Distance metric learning with application to clustering with
  side-information.
\newblock In \emph{Advances in Neural Information Processing Systems}, pages
  505--512, 2002.

\bibitem[Yang and Jin(2006)]{yang2006distance}
Liu Yang and Rong Jin.
\newblock Distance metric learning: A comprehensive survey.
\newblock \emph{Michigan State Universiy}, 2, 2006.

\end{thebibliography}

%
%
%
%
%
%
%
%
%
%
%
%

\end{document}